\documentclass[letterpaper, 10 pt, conference]{ieeeconf}

\IEEEoverridecommandlockouts
\usepackage{amsmath,amssymb,amsfonts}
\usepackage{mathtools}
\usepackage{graphicx}
\usepackage{booktabs}
\usepackage{bbm}
\usepackage{url}
\usepackage{tikz}
\usepackage{pgfplots}
\pgfplotsset{compat=1.18}
\usepackage{multirow}
\usepackage{xcolor}

\newtheorem{theorem}{Theorem}
\newtheorem{assumption}{Assumption}
\newtheorem{proposition}{Proposition}

\newcommand{\E}{\mathbb{E}}

\title{\LARGE \bf
Human Supervision as an Information Bottleneck: A Unified Theory of Error Floors in Human-Guided Learning
}

\author{
Alejandro Rodriguez Dominguez%
\thanks{Alejandro Rodriguez Dominguez is Head of Quantitative Analysis at Miralta Finance Bank S.A., Spain, 
and a PhD Candidate with the Department of Computer Science, University of Reading, UK.
{\tt\small aodriguez@miraltabank.com}}%
}

\begin{document}

\maketitle
\thispagestyle{empty}
\pagestyle{empty}

\begin{abstract}
Large language models are trained primarily on human-generated data and feedback, yet they exhibit persistent errors arising from annotation noise, subjective preferences, and the limited expressive bandwidth of natural language. We argue that these limitations reflect structural properties of the supervision channel rather than model scale or optimization. We develop a unified theory showing that whenever the human supervision channel is not sufficient for a latent evaluation target, it acts as an information-reducing channel that induces a strictly positive excess-risk floor for any learner dominated by it. We formalize this Human-Bounded Intelligence limit and show that across six complementary frameworks—operator theory, PAC-Bayes, information theory, causal inference, category theory, and game-theoretic analyses of reinforcement learning from human feedback—non-sufficiency yields strictly positive lower bounds arising from the same structural decomposition into annotation noise, preference distortion, and semantic compression. The theory explains why scaling alone cannot eliminate persistent human-aligned errors and characterizes conditions under which auxiliary non-human signals (e.g., retrieval, program execution, tools) increase effective supervision capacity and collapse the floor by restoring information about the latent target. Experiments on real preference data, synthetic known-target tasks, and externally verifiable benchmarks confirm the predicted structural signatures: human-only supervision exhibits a persistent floor, while sufficiently informative auxiliary channels strictly reduce or eliminate excess error.
\end{abstract}

\section{Introduction}

Large language models (LLMs) are trained on human-generated
text \cite{radford2019language}, refined via Reinforcement Learning
from Human Feedback (RLHF) \cite{christiano2017deep,ouyang2022training},
and evaluated according to human preferences. Despite their
capabilities, such systems inherit limitations of human supervision,
including annotation noise, shortcut biases, and subjective
distortions \cite{geirhos2020shortcut}. These induce persistent error
patterns reflecting human preferences and semantic bottlenecks,
raising a central question:

\begin{center}
\emph{Can a system trained solely on human-generated signals reliably
exceed performance relative to the underlying task objective?}
\end{center}

Empirically, LLMs outperform humans on narrow benchmarks, yet
pipelines relying exclusively on human-labeled or human-judged data
exhibit reward hacking \cite{skalse2022defining}, preference drift and
overoptimization \cite{gao2023scaling}, and degradation under iterative
self-training \cite{shumailov2023curse}. These effects often persist
with scale, suggesting a structural rather than optimization-based
limitation.

We provide a general explanation. Under mild assumptions on (i) human
signal generation, (ii) empirical risk minimization, and (iii)
information flow through the supervision pipeline, any learner
dominated by human-generated signals is constrained by a strictly
positive \emph{excess-risk floor}. Human supervision—whether labels,
rankings, demonstrations, or curated text—acts as an
information-reducing channel of the latent task objective. Even with
unlimited capacity, infinite data, and ideal optimization, a learner
cannot recover information that never passes through this channel.

We formalize this Human-Bounded Intelligence (HBI) limit across six
theoretical perspectives: operator theory, PAC-Bayes
\cite{mcallester1999pac}, information theory \cite{cover1999elements},
causal inference \cite{pearl2009causality}, category theory
\cite{maclane1998categories}, and game-theoretic RLHF models
\cite{christiano2017deep}. In each framework, non-sufficiency of the
supervision channel yields a strictly positive lower bound of the form
\begin{equation}
\liminf_{n\to\infty}\mathcal{E}^{\ast}(f_{\theta_n})
\ge \gamma_H^{(\cdot)} > 0,
\end{equation}
where the constant depends on properties of the supervision channel
but not on model scale or compute. Although numerical bounds differ
across frameworks, their positivity arises from the same structural
source: information loss induced by $P_H$. Moreover, supervision
induces a structured decomposition
\begin{equation}
B_H = B_{\mathrm{noise}} + B_{\mathrm{pref}} + B_{\mathrm{sem}},
\end{equation}
corresponding to annotation noise, preference distortion, and
semantic compression.

Empirically, we validate these predictions across three regimes:
real human-preference data, synthetic tasks with known ground truth,
and externally verifiable benchmarks. Human-only supervision exhibits
a persistent floor; hybrid supervision with auxiliary signals that
provide independent information about $Y^\ast$ reduces the floor and
eliminates it when the auxiliary channel becomes sufficient.

The HBI limit applies only when human supervision is the primary
information source about $Y^\ast$. Auxiliary non-human channels—
including code execution \cite{chen2021evaluating}, retrieval
\cite{lewis2020retrieval}, and tool augmentation
\cite{schick2023toolformer}—increase effective supervision capacity and
can remove the floor by restoring information about the latent target.
Hybrid systems therefore modify the supervision channel rather than
merely scaling model capacity.

Our contributions are:
\begin{itemize}
\item A unified framework modeling human supervision as an
information-reducing channel with a structured bias decomposition.
\item A Human-Bounded Intelligence (HBI) theorem establishing a strictly
positive excess-risk floor under human-dominated supervision.
\item Instantiations across six independent theoretical frameworks
demonstrating the same structural limitation.
\item A characterization of auxiliary channels that break the bound,
defining human-only, human+model, and human+model+auxiliary regimes.
\item Empirical validation confirming floor persistence under
human-only supervision and collapse under sufficient auxiliary
information.
\end{itemize}

The remainder of the paper is organized as follows.
Section~\ref{sec:related} reviews related work.
Section~\ref{sec:framework} presents the formal framework and HBI theorem.
Section~\ref{sec:instantiations} derives lower bounds across six perspectives.
Section~\ref{sec:breaking} characterizes auxiliary-channel regimes.
Section~\ref{sec:experiments} provides empirical validation.
Section~\ref{sec:discussion} concludes.

\section{Related Work}
\label{sec:related}

Large language models trained on human-curated corpora
\cite{radford2019language} and refined via RLHF
\cite{christiano2017deep,ouyang2022training}
inherit limitations of human supervision, including annotation
noise, shortcut bias, and preference distortion
\cite{geirhos2020shortcut}. Empirical studies of reward
overoptimization and objective misspecification
\cite{gao2023scaling,skalse2022defining}
demonstrate systematic gaps between learned reward models and
latent task objectives, motivating theoretical accounts of
supervision-induced bias.

Learning with noisy labels
\cite{natarajan2013learning}
characterizes corruption in supervision signals, and shortcut
learning analyses \cite{geirhos2020shortcut} show exploitation of
spurious human-annotated features. These works typically model
specific perturbations. In contrast, we treat human supervision
itself as an information-reducing channel and derive a structural
decomposition into annotation noise, preference distortion, and
semantic compression independent of any particular corruption
mechanism.

Scaling-law analyses
\cite{kaplan2020scaling,hoffmann2022training}
document smooth improvements with data and compute, while
model-collapse results \cite{shumailov2023curse}
show degradation under recursive self-training. Our results
complement these findings: even with infinite data and ideal
optimization, a supervision channel of bounded information
capacity induces a nonzero excess-risk floor that scaling alone
cannot remove.

Inverse reward learning \cite{ng2000algorithms}
formalizes ambiguity among reward functions consistent with
observed feedback. Our causal and information-theoretic
instantiations connect this ambiguity to supervision-channel
non-invertibility, yielding lower bounds independent of model
class or estimation procedure.

Tool- and retrieval-augmented systems
\cite{schick2023toolformer,chen2021evaluating,lewis2020retrieval}
introduce auxiliary evaluators such as program execution and
search, supplying additional information about latent targets.
Our auxiliary-channel analysis formalizes when such signals
strictly increase mutual information and eliminate the
human-bounded floor.

More broadly, our formulation draws on information theory
\cite{cover1999elements}, causal identifiability
\cite{pearl2009causality}, PAC-Bayesian generalization
\cite{mcallester1999pac}, and categorical structure
\cite{maclane1998categories}. We differ from prior work by
establishing a supervision-induced excess-risk floor and showing
that the same structural limitation arises independently across
multiple theoretical frameworks.

\section{Formal Framework and HBI Theorem}
\label{sec:framework}

We model supervision as a stochastic transformation from an unknown latent task-relevant mapping $f^\ast:\mathcal{X}\to\mathcal{Y}$ to a human-provided signal $S$, with the learner observing only $(X,S)$. The analysis does not assume metaphysical objectivity: $Y^\ast$ may represent any latent evaluation target not fully revealed by human supervision.

\subsection{Ground Truth, Human Channel, and Bias}

Let $(\mathcal{X},\mathcal{A},\mu)$ be a probability space and let $f^\ast:\mathcal{X}\to\mathcal{Y}$ be a measurable mapping. Define the latent evaluation variable
\begin{equation}
    Y^\ast = f^\ast(X).
\end{equation}

Human supervision is generated through a stochastic channel
\begin{equation}
    S \sim P_H(\cdot \mid X,Y^\ast),
\end{equation}
which need not be sufficient for $Y^\ast$. In additive settings (e.g., $\mathcal{Y}\subseteq\mathbb{R}$ with squared loss), one illustrative decomposition is
\begin{equation}
\label{eq:signal-decomp}
S = Y^\ast + \varepsilon_H + b_H(X) + (q_H(Y^\ast)-Y^\ast),
\end{equation}
where:
\begin{itemize}
\item $\varepsilon_H$ is zero-mean stochastic noise,
\item $b_H$ is systematic preference distortion,
\item $q_H$ is a possibly non-invertible semantic compression.
\end{itemize}
This decomposition is heuristic; the theory below does not rely on additivity. Let $\ell$ denote the ground-truth loss and $L$ the surrogate optimized during training. Define the population risks
\begin{equation}
\mathcal{R}^\ast(f)
=
\mathbb{E}[\ell(f(X),Y^\ast)],
\qquad
\mathcal{R}_H(f)
=
\mathbb{E}[L(f(X),S)].
\end{equation}

Define the induced human bias functional
\begin{equation}
B_H(f)
=
\mathcal{R}_H(f)-\mathcal{R}^\ast(f).
\end{equation}
This functional depends only on the supervision channel $P_H$ and the surrogate $L$.

\subsection{Learners and Assumptions}

Let $\mathcal{F}\subseteq \mathcal{L}^2(\mu)$ be a hypothesis class. Consider empirical risk minimization
\begin{equation}
\hat{\theta}_n
\in
\arg\min_\theta
\frac{1}{n}\sum_{i=1}^n
L(f_\theta(X_i),S_i).
\end{equation}

We impose the following structural assumptions.

\begin{assumption}[Human-dominated supervision]
\label{ass:human-only}
All information about $Y^\ast$ available during training flows through the channel $P_H(\cdot\mid X,Y^\ast)$. No auxiliary signal provides additional independent information about $Y^\ast$.
\end{assumption}

\begin{assumption}[Asymptotically optimal optimization]
\label{ass:ideal-opt}
\[
\lim_{n\to\infty}
\mathcal{R}_H(f_{\hat{\theta}_n})
=
\inf_{f\in\mathcal{F}}
\mathcal{R}_H(f)
\eqqcolon
R^H_{\mathrm{opt}}.
\]
\end{assumption}

\begin{assumption}[Strict minimizer separation]
\label{ass:min-sep}
The minimizers of $\mathcal{R}_H$ do not coincide with those of $\mathcal{R}^\ast$. Define
\[
\gamma_H
\coloneqq
\inf_{f\in\arg\min \mathcal{R}_H}
\big(
\mathcal{R}^\ast(f)-R^\ast_{\mathrm{opt}}
\big).
\]
Then $\gamma_H>0$.
\end{assumption}

\paragraph*{Regularity Conditions}
We assume:
\begin{enumerate}
\item $\mathcal{R}_H$ and $\mathcal{R}^\ast$ attain their minima over $\mathcal{F}$;
\item $\mathcal{R}^\ast$ is lower semicontinuous under the topology in which predictors converge;
\item Any sequence $\{f_n\}$ satisfying $\mathcal{R}_H(f_n)\to R^H_{\mathrm{opt}}$ has accumulation points contained in $\arg\min \mathcal{R}_H$.
\end{enumerate}

These conditions are standard in statistical learning theory and isolate structural limitations of the supervision channel from optimization artifacts.

\subsection{Excess Risk and the Human-Bounded Limit}

Define the ground-truth excess risk
\begin{equation}
\mathcal{E}^\ast(f)
=
\mathcal{R}^\ast(f)-R^\ast_{\mathrm{opt}},
\qquad
R^\ast_{\mathrm{opt}}
=
\inf_{g\in\mathcal{F}}
\mathcal{R}^\ast(g).
\end{equation}

A learner is \emph{human-bounded} if
\[
\liminf_{n\to\infty}
\mathcal{E}^\ast(f_{\hat{\theta}_n})
>
0.
\]

\begin{theorem}[Human-Bounded Intelligence (HBI)]
\label{thm:HBI}
Under Assumptions~\ref{ass:human-only}--\ref{ass:min-sep} and the regularity conditions,
\[
\liminf_{n\to\infty}
\mathcal{E}^\ast(f_{\hat{\theta}_n})
\ge
\gamma_H
>
0.
\]
\end{theorem}

\subsection{Proof of Theorem~\ref{thm:HBI}}

\begin{proof}
By Assumption~\ref{ass:ideal-opt},
\[
\mathcal{R}_H(f_{\hat{\theta}_n})
\to
R^H_{\mathrm{opt}}.
\]

By regularity condition (3), every accumulation point of 
$\{f_{\hat{\theta}_n}\}$ lies in $\arg\min \mathcal{R}_H$. Let $f_\infty$ be any such accumulation point. Then
\[
f_\infty \in \arg\min \mathcal{R}_H.
\]

By Assumption~\ref{ass:min-sep},
\[
\mathcal{R}^\ast(f_\infty)
-
R^\ast_{\mathrm{opt}}
\ge
\gamma_H.
\]

Lower semicontinuity of $\mathcal{R}^\ast$ yields
\[
\liminf_{n\to\infty}
\mathcal{R}^\ast(f_{\hat{\theta}_n})
\ge
\mathcal{R}^\ast(f_\infty).
\]

Therefore,
\[
\liminf_{n\to\infty}
\mathcal{E}^\ast(f_{\hat{\theta}_n})
\ge
\gamma_H.
\]
\end{proof}

\section{Instantiations Across Six Frameworks}
\label{sec:instantiations}

We now show that six classical frameworks independently yield
strictly positive excess-risk lower bounds under non-sufficiency
of the supervision channel. In each case, the bound arises from
the same structural decomposition of supervision-induced deviation,
\begin{equation}
\gamma_H^{(\cdot)}
=
\gamma_{\mathrm{noise}}^{(\cdot)}
+
\gamma_{\mathrm{pref}}^{(\cdot)}
+
\gamma_{\mathrm{sem}}^{(\cdot)},
\end{equation}
where the constants may differ across analytical perspectives.
Although the numerical lower bounds obtained in each framework
need not coincide, their positivity originates from the same
structural source: information loss induced by $P_H$.

\subsection{Operator-Theoretic Limit}

We model the ground truth as a bounded linear operator
\begin{equation}
    T^\ast : L^2(\mathcal{X},\mu) \to L^2(\mathcal{Y},\nu),
\end{equation}
mapping input functions to output functions. Human supervision induces an operator $T_H$ capturing what is representable via $P_H$, and
\begin{equation}
B_H = T_H - T^\ast.
\end{equation}
Under ideal optimization, many infinite-width or kernel-limit analyses imply
convergence $T_{\theta_n} \to T_H$
\cite{belkin2018understand,arora2019fine,jacot2018ntk,lee2019wide}.

\begin{theorem}[Operator-Theoretic HBI]
\label{thm:operator}
Suppose $T_{\theta_n} \to T_H$ in operator norm and $B_H\neq 0$. Then
\begin{equation}
    \lim_{n\to\infty} \big\|T_{\theta_n} - T^\ast\big\|
    = \|B_H\|
\end{equation}
and, if the task excess risk is Lipschitz in $\|T-T^\ast\|$, there exists a
task-dependent constant $c>0$ such that
\begin{equation}
    \gamma_H \;\ge\; c\,\|B_H\| > 0.
\end{equation}
\end{theorem}

\begin{proof}
We have
\begin{equation}
T_{\theta_n}-T^\ast
=
(T_{\theta_n}-T_H) + B_H.
\end{equation}
Taking norms and applying the triangle inequality,
\begin{equation}
\| T_{\theta_n} - T^\ast \|
\ge
\| B_H \| - \|T_{\theta_n}-T_H\|,
\end{equation}
so
\(
\liminf_{n\to\infty}\|T_{\theta_n}-T^\ast\|
\ge\|B_H\|.
\)
Similarly,
\begin{equation}
\| T_{\theta_n} - T^\ast \|
\le
\| B_H \| + \|T_{\theta_n}-T_H\|,
\end{equation}
so
\(
\limsup_{n\to\infty}\|T_{\theta_n}-T^\ast\|
\le\|B_H\|.
\)
Thus, the limit exists and equals $\|B_H\|$. If the excess risk $\mathcal{E}^\ast$ satisfies $\mathcal{E}^\ast(f_\theta)\ge c\|T_\theta-T^\ast\|$ for some $c>0$ (for example, for squared loss under an $L^2$–Lipschitz condition), the lower bound on $\gamma_H$ follows.
\end{proof}

\subsection{PAC-Bayesian Limit}

PAC-Bayes bounds \cite{mcallester1999pac,seeger2002pac,guedj2019primer} relate expected risk to empirical risk plus a complexity penalty. Let $P$ be a prior and $Q$ a posterior over hypotheses. Write
\begin{equation}
L_H(f) = L^\ast(f) + B_H(f).
\end{equation}
A standard PAC-Bayes inequality gives, with probability at least $1-\delta$,
\begin{equation}
\mathbb{E}_{f\sim Q}[L_H(f)]
\le
\mathbb{E}_{f\sim Q}[\hat L_H(f)]
+
\sqrt{\frac{\mathrm{KL}(Q\|P)+\ln(1/\delta)}{2n}}.
\end{equation}

\begin{theorem}[PAC-Bayes HBI]
\label{thm:pacbayes}
Let $Q_H^{(n)}$ be posteriors that concentrate on minimizers of $L_H$ as
$n\to\infty$, and assume that any minimizer $f_H^\ast$ of $L_H$ satisfies
\begin{equation}
L^\ast(f_H^\ast) \ge \inf_f L^\ast(f) + \gamma_H^{\mathrm{PAC}}
\end{equation}
for some $\gamma_H^{\mathrm{PAC}}>0$. Then
\begin{equation}
\liminf_{n\to\infty} \mathbb{E}_{f\sim Q_H^{(n)}}[L^\ast(f)]
\;\ge\;
\inf_f L^\ast(f)
+ \gamma_H^{\mathrm{PAC}}.
\end{equation}
\end{theorem}

\begin{proof}
By construction, $Q_H^{(n)}$ concentrates on the set of minimizers of $L_H$,
so
\(
\lim_{n\to\infty}\E_{f\sim Q_H^{(n)}}[L_H(f)]=\inf_f L_H(f).
\)
By the assumption on minimizers, any such $f_H^\ast$ obeys
\(
L^\ast(f_H^\ast)\ge \inf_f L^\ast(f)+\gamma_H^{\mathrm{PAC}}.
\)
Since $L^\ast$ is continuous in an appropriate topology and
$Q_H^{(n)}\Rightarrow \delta_{f_H^\ast}$ along some subsequence, we obtain
\(
\liminf_{n\to\infty}\E_{f\sim Q_H^{(n)}}[L^\ast(f)]
\ge \inf_f L^\ast(f)+\gamma_H^{\mathrm{PAC}}.
\)
\end{proof}

This instantiation corresponds to the case where the human-aligned Gibbs posterior does not concentrate on ground-truth minimizers.

\subsection{Information-Theoretic Limit}

Under Assumption~\ref{ass:human-only}, the supervision pipeline induces the Markov chain
\[
Y^\ast \to S \to \Theta.
\]
By the data-processing inequality,
\begin{equation}
I(Y^\ast;\Theta)
\le
I(Y^\ast;S)
=
C_H^{\mathrm{eff}}.
\end{equation}

Let $d:\mathcal{Y}\times\mathcal{Y}\to\mathbb{R}_+$ be a distortion measure compatible with the task loss and define
\[
D_\Theta
\coloneqq
\mathbb{E}[d(Y^\ast,\hat Y_\Theta)].
\]
Let $R(D)$ denote the rate–distortion function of $Y^\ast$ under $d$. By rate–distortion theory,
\begin{equation}
I(Y^\ast;\Theta)
\ge
R(D_\Theta).
\end{equation}

Combining with data processing,
\begin{equation}
R(D_\Theta)
\le
C_H^{\mathrm{eff}}.
\end{equation}

\begin{theorem}[Information-Theoretic HBI]
\label{thm:infotheory}
Assume $R(D)$ is strictly decreasing and continuous on the interval of interest.
If
\[
C_H^{\mathrm{eff}}
<
R(D^\ast),
\]
where $D^\ast$ is the minimal achievable distortion under full information, then
\begin{equation}
\liminf_{n\to\infty}
D_\Theta
\ge
R^{-1}(C_H^{\mathrm{eff}})
\eqqcolon
D_H
>
D^\ast.
\end{equation}
If the task excess risk satisfies
\[
\mathcal{E}^\ast(f)
\ge
c\,(D_\Theta - D^\ast)
\quad\text{for some } c>0,
\]
then
\[
\gamma_H \ge c\,(D_H - D^\ast) > 0.
\]
\end{theorem}

\begin{proof}
Rate–distortion theory implies
\(
R(D_\Theta)\le I(Y^\ast;\Theta).
\)
Data processing gives
\(
I(Y^\ast;\Theta)\le C_H^{\mathrm{eff}}.
\)
Hence
\(
R(D_\Theta)\le C_H^{\mathrm{eff}}.
\)

Since $R(D)$ is strictly decreasing and continuous, it is invertible on its image, yielding
\(
D_\Theta\ge R^{-1}(C_H^{\mathrm{eff}}).
\)
If $C_H^{\mathrm{eff}}<R(D^\ast)$, strict monotonicity implies
\(
R^{-1}(C_H^{\mathrm{eff}})>D^\ast.
\)
The excess-risk lower bound follows from the distortion–risk inequality.
\end{proof}

\subsection{Causal Non-Identifiability}

In a Structural Causal Model (SCM) \cite{pearl2009causality,peters2017elements},
\begin{equation}
X \rightarrow Y^\ast \rightarrow S \rightarrow \Theta, \qquad H \rightarrow S,
\end{equation}
the supervision mechanism $S = h(X, Y^\ast, H, U_S)$ is typically many-to-one in $Y^\ast$: different ground-truth outputs may receive the same human label or judgment. This non-invertibility makes $f^\ast$ non-identifiable from $(X,S)$.

\begin{assumption}[Human channel non-invertibility]
\label{ass:noninvertible}
There exists a measurable set $A\subseteq\mathcal{X}$ with positive probability such that for all $x\in A$ there exist $y_1^\ast\neq y_2^\ast$ with
\begin{equation}
    P_H(\cdot \mid X=x, Y^\ast=y_1^\ast)
    =
    P_H(\cdot \mid X=x, Y^\ast=y_2^\ast).
\end{equation}
\end{assumption}

\begin{theorem}[Causal HBI]
\label{thm:causal}
Under Assumption~\ref{ass:noninvertible}, the ground-truth mapping $f^\ast$ is not identifiable from human-supervised data on $A$. For any estimator
sequence $f_{\hat{\theta}_n}$ based on $(X,S)$, there exists a compatible SCM in which the ground-truth excess risk is bounded below by the Bayes risk on $A$, hence strictly positive whenever $Y^\ast$ is nondegenerate on $A$.
\end{theorem}

\begin{proof}
Non-invertibility implies that for $x\in A$ and any estimator based on $(X,S)$, one cannot distinguish $y_1^\ast$ from $y_2^\ast$. For 0–1 loss, the Bayes optimal classifier on $A$ for two such labels has an error of at least
$\min\{p,1-p\}$, where $p$ is the conditional probability of $Y^\ast=y_1^\ast$. Thus, the excess risk on $A$ is bounded below by a positive constant whenever both labels occur with nonzero probability. For general losses, the same argument applies with the corresponding Bayes risk.
\end{proof}

\subsection{Category-Theoretic Formulation}

Let $\mathcal{C}$ denote a category of semantic task objects and $\mathcal{H}$ a category of human-representable structures.
Human supervision induces a functor
\[
F_H : \mathcal{C} \to \mathcal{H}.
\]

Let $L : \mathcal{C} \to \mathbf{R}$ be an evaluation functor,
where $\mathbf{R}$ is viewed as a category whose objects are real numbers
and whose morphisms are identities.

Define the equivalence relation
\[
c_1 \sim c_2
\quad\Longleftrightarrow\quad
F_H(c_1)=F_H(c_2).
\]

\begin{theorem}[No-Factorization Lower Bound]
\label{thm:categorical}
The evaluation functor $L$ factors through $F_H$
(i.e., $L=\tilde L\circ F_H$ for some functor $\tilde L$)
if and only if $L$ is constant on each equivalence class of $\sim$. If there exist $c_1\sim c_2$ with $L(c_1)\neq L(c_2)$,
then any predictor depending only on $F_H(c)$ incurs irreducible excess loss at least
\[
\gamma_{\mathrm{sem}}
=
\frac12 |L(c_1)-L(c_2)|.
\]
\end{theorem}

\begin{proof}
Factorization $L=\tilde L\circ F_H$ holds if and only if
\(
L(c_1)=L(c_2)
\)
whenever
\(
F_H(c_1)=F_H(c_2),
\)
which is the universal property of quotient objects.

If $L$ is not constant on some equivalence class,
there exist $c_1\sim c_2$ with distinct loss.
Any predictor defined solely on $\mathcal{H}$ must assign identical outputs to $c_1$ and $c_2$.
Therefore at least one of the two incurs loss at least half their difference, establishing the bound.
\end{proof}

\subsection{RLHF as a Biased Fixed Point}

Preference-based methods such as RLHF
\cite{christiano2017deep,ouyang2022training}
optimize a human-aligned utility
\begin{equation}
U_H(\pi)
=
U^\ast(\pi)
+
B_H(\pi),
\end{equation}
where
\(
B_H(\pi)=\mathbb{E}_{y\sim\pi}[B_H(y)]
\). Let
\[
\pi^\ast \in \arg\max_\pi U^\ast(\pi),
\qquad
\pi_H^\ast \in \arg\max_\pi U_H(\pi).
\]

\begin{theorem}[Biased Optimization Gap]
\label{thm:biased-nash}
Assume:
\begin{enumerate}
\item The policy space is compact and $U^\ast$, $B_H$ are continuous;
\item The mapping $\pi \mapsto B_H(\pi)$ is not constant.
\end{enumerate}
Then
\[
U^\ast(\pi^\ast)
-
U^\ast(\pi_H^\ast)
>0.
\]
\end{theorem}

\begin{proof}
Suppose equality holds.
Then $\pi_H^\ast$ also maximizes $U^\ast$.
Since it maximizes $U_H$,
\[
U^\ast(\pi_H^\ast)+B_H(\pi_H^\ast)
\ge
U^\ast(\pi)+B_H(\pi)
\quad \forall \pi.
\]
Setting $\pi=\pi^\ast$ and using optimality of $\pi^\ast$ for $U^\ast$ yields
\[
B_H(\pi_H^\ast) \ge B_H(\pi^\ast).
\]
By symmetry of the argument,
\(
B_H(\pi_H^\ast)=B_H(\pi^\ast).
\)
Repeating for arbitrary $\pi$ implies $B_H$ must be constant over the policy space,
contradicting assumption (2).
\end{proof}

\section{Breaking the Human-Bounded Limit}
\label{sec:breaking}

The HBI theorem applies only under Assumption~1, namely when all
information about $Y^\ast$ available during training flows exclusively
through the human supervision channel $P_H$. If an auxiliary channel
provides additional information about $Y^\ast$ beyond what is contained
in $S_H$, the excess-risk floor can strictly decrease or vanish.
The limitation is therefore structural but conditional: it depends on
the information geometry of the supervision pipeline.

\subsection{Auxiliary Information and Hybrid Supervision}

Let $S_A$ denote an auxiliary signal and $S_M$ a model-generated signal.
Hybrid supervision can be represented abstractly as
\begin{equation}
S_{\mathrm{mix}}
= F_{\alpha,\beta,\gamma}(S_H,S_M,S_A),
\qquad
\alpha+\beta+\gamma=1,
\label{eq:mix-schematic}
\end{equation}
where $F$ denotes an arbitrary measurable combination rule (e.g.,
weighted mixing, joint embeddings, multi-objective optimization, or
concatenation). This representation is schematic and does not assume
linearity.

By the chain rule for mutual information,
\begin{align}
I(Y^\ast; S_H,S_M,S_A)
&=
I(Y^\ast;S_H)
+ I(Y^\ast;S_M \mid S_H)
\nonumber \\
&\quad
+ I(Y^\ast;S_A \mid S_H,S_M)
\eqqcolon C_{\mathrm{mix}},
\label{eq:capacity-mix}
\end{align}
where $C_{\mathrm{mix}}$ denotes the effective supervision capacity under
hybrid signals. This decomposition requires no independence assumptions;
it simply quantifies how additional channels increase information about
$Y^\ast$.

\begin{proposition}[Auxiliary Channel Reduces the Floor]
\label{prop:aux-break}
Suppose $I(Y^\ast; S_A \mid S_H) > 0$.
Then
\[
C_{\mathrm{mix}} > I(Y^\ast; S_H),
\]
and the induced excess-risk floor under hybrid supervision,
denoted $\gamma_{H+A}^{(\cdot)}$, satisfies
\[
\gamma_{H+A}^{(\cdot)} \le \gamma_H^{(\cdot)},
\]
with strict inequality whenever the rate--distortion function is
strictly decreasing on the relevant interval.
Moreover, if $S_A$ is sufficient for $Y^\ast$ (i.e.,
$I(Y^\ast;S_A \mid S_H)=H(Y^\ast\mid S_H)$),
then the floor collapses to zero.
\end{proposition}

\begin{proof}
From \eqref{eq:capacity-mix},
\[
C_{\mathrm{mix}}
=
I(Y^\ast;S_H)
+
I(Y^\ast;S_M \mid S_H)
+
I(Y^\ast;S_A \mid S_H,S_M).
\]
Hence $C_{\mathrm{mix}} > I(Y^\ast;S_H)$ whenever
$I(Y^\ast;S_A \mid S_H)>0$.

By rate--distortion monotonicity, increasing channel capacity weakly
reduces the minimal achievable distortion. Since excess risk is
lower-bounded by distortion gap in the information-theoretic
instantiation, the induced floor under hybrid supervision cannot exceed
that under human-only supervision, and is strictly smaller under strict
monotonicity. If $S_A$ renders $Y^\ast$ identifiable up to Bayes-optimal
distortion, the achievable distortion equals the optimal distortion,
and the excess-risk floor vanishes.
\end{proof}

Auxiliary information therefore mitigates $\gamma_{\mathrm{noise}}^{(\cdot)}$,
reduces $\gamma_{\mathrm{pref}}^{(\cdot)}$, and can eliminate
$\gamma_{\mathrm{sem}}^{(\cdot)}$ whenever it resolves structure not
expressible through $P_H$. The bound is not architectural but
informational: it disappears precisely when the supervision channel
becomes sufficient.

\subsection{Supervision Regimes}

Three regimes follow directly from this analysis.

\paragraph*{Human-only (H)}
When supervision is dominated by $S_H$, the induced floor satisfies
\[
\gamma_H^{(\cdot)}
=
\gamma_{\mathrm{noise}}^{(\cdot)}
+
\gamma_{\mathrm{pref}}^{(\cdot)}
+
\gamma_{\mathrm{sem}}^{(\cdot)},
\]
and the learner converges to the human-aligned solution.

\paragraph*{Hybrid Human+Model (H+M)}
Model-generated signals can reduce variance and partially correct
annotation noise but do not introduce fundamentally new information
about $Y^\ast$. Consequently, the induced floor satisfies
\[
\gamma_{H+M}^{(\cdot)}
\le
\gamma_H^{(\cdot)},
\]
with noise effects attenuated while structural preference and semantic
distortions may remain.

\paragraph*{Hybrid with Auxiliary Channels (H+M+A)}
When auxiliary channels provide independent information about $Y^\ast$,
semantic compression can be removed and preference-induced gaps reduced,
yielding
\[
\gamma_{H+M+A}^{(\cdot)}
\le
\gamma_{H+M}^{(\cdot)}
\le
\gamma_H^{(\cdot)},
\]
with $\gamma_{H+M+A}^{(\cdot)}=0$ in the sufficient-channel case.

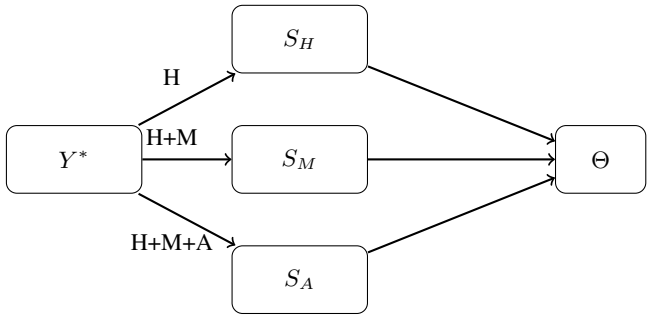
\begin{figure}[t]
\centering
\begin{tikzpicture}[
  font=\small,
  box/.style={draw, rounded corners, minimum width=1.8cm, minimum height=0.9cm, align=center}
]
\node[box] (Y)    at (0,0) {$Y^\ast$};
\node[box] (SH)   at (3,1.6) {$S_H$};
\node[box] (SM)   at (3,0)   {$S_M$};
\node[box] (SA)   at (3,-1.6) {$S_A$};
\node[box, minimum width=1.2cm] (Theta) at (7,0) {$\Theta$};

\draw[->, thick] (Y) -- (SH);
\draw[->, thick] (Y) -- (SM);
\draw[->, thick] (Y) -- (SA);
\draw[->, thick] (SH) -- (Theta);
\draw[->, thick] (SM) -- (Theta);
\draw[->, thick] (SA) -- (Theta);

\node at (1.3,1.1) {\small H};
\node at (1.3,0.3) {\small H+M};
\node at (1.3,-1.1) {\small H+M+A};
\end{tikzpicture}
\caption{Conceptual information flow under human-only (H), hybrid human+model (H+M), and hybrid with auxiliary channels (H+M+A). Auxiliary channels introduce additional information about $Y^\ast$, increasing effective supervision capacity and reducing or eliminating the structural excess-risk floor.}
\label{fig:hybrid-diagram}
\end{figure}

\section{Experiments}
\label{sec:experiments}

We empirically evaluate the structural predictions of the Human-Bounded Intelligence (HBI) theorem across three complementary regimes: (i) real human-preference data, (ii) controlled synthetic tasks with known ground truth, and (iii) externally verifiable objective benchmarks. We test whether hybrid supervision ($\alpha<1$) improves generalization, corruption robustness, scaling behavior, and distortion when $R^\ast$ is known. All metrics report pairwise accuracy on held-out comparisons with 95\% confidence intervals across three seeds.

\subsection{Real Preference Data}

Experiments use \texttt{Dahoas/full-hh-rlhf} \cite{dahoas2021hh} with 1000 training and 300 test pairs unless otherwise stated. Reward models (BERT-tiny, DistilRoBERTa-base, RoBERTa-base) are trained using Bradley--Terry loss. An auxiliary verifier (TinyLlama-1.1B-Chat-v1.0, 4-bit) provides $S_A(x,y) = -\mathrm{NLL}_{\mathrm{LLM}}(x,y)$, and hybrid scoring is defined as $S_{\alpha,\lambda} = \alpha S_M + (1-\alpha)\lambda S_A$.

\paragraph*{$\alpha$-Sweep}
Sweeping $\alpha\in\{0,0.25,0.5,0.75,1\}$ shows that human-only supervision ($\alpha=1$) is never optimal (Table~\ref{tab:alpha-summary}). Hybrid gains range from +0.007 to +0.059 and are largest in lower-capacity models, consistent with structural bottleneck effects predicted by HBI.

\begin{table}[t]
\centering
\footnotesize
\begin{tabular}{c|c|c}
\toprule
Model & Human-only & Best Hybrid \\
\midrule
BERT-tiny & 0.399 $\pm$ 0.022 & 0.458 $\pm$ 0.031 \\
DistilRoBERTa & 0.433 $\pm$ 0.006 & 0.440 $\pm$ 0.043 \\
RoBERTa-base & 0.441 $\pm$ 0.018 & 0.465 $\pm$ 0.021 \\
\bottomrule
\end{tabular}
\caption{Human-only versus best hybrid supervision}
\label{tab:alpha-summary}
\end{table}

\paragraph*{$\lambda$-Ablation}
Fixing $\alpha=0.5$, we sweep $\lambda\in\{0.5,1.0,2.0\}$. Results (Table~\ref{tab:lambda}) remain stable across $\lambda$, indicating that improvements arise from structural mixing rather than hyperparameter tuning.

\begin{table}[t]
\centering
\footnotesize
\begin{tabular}{c|ccc}
\toprule
Model & $\lambda=0.5$ & $\lambda=1.0$ & $\lambda=2.0$ \\
\midrule
BERT-tiny & 0.451 & 0.458 & 0.451 \\
DistilRoBERTa & 0.437 & 0.436 & 0.431 \\
RoBERTa-base & 0.462 & 0.465 & 0.463 \\
\bottomrule
\end{tabular}
\caption{$\lambda$-ablation}
\label{tab:lambda}
\end{table}

\paragraph*{Noise Robustness}
Human labels are flipped with probability $\gamma\in\{0,0.2,0.4\}$. As shown in Table~\ref{tab:noise}, hybrid supervision consistently mitigates degradation under corruption, reducing the noise component $\gamma_{\mathrm{noise}}$ predicted by the theoretical decomposition.

\begin{table}[t]
\centering
\footnotesize
\begin{tabular}{c|c|cc}
\toprule
Model & $\gamma$ & Human-only & Hybrid \\
\midrule
\multirow{3}{*}{BERT-tiny}
 & 0.0 & 0.399 & 0.458 \\
 & 0.2 & 0.392 & 0.462 \\
 & 0.4 & 0.389 & 0.462 \\
\midrule
\multirow{3}{*}{DistilRoBERTa}
 & 0.0 & 0.431 & 0.436 \\
 & 0.2 & 0.430 & 0.433 \\
 & 0.4 & 0.394 & 0.437 \\
\bottomrule
\end{tabular}
\caption{Robustness under human corruption}
\label{tab:noise}
\end{table}

\paragraph*{Scaling Behavior}
Training size varies over $N\in\{2000,4000,8000,16000\}$. Results in Table~\ref{tab:scaling} and Fig.~\ref{fig:real-scaling} show scaling reduces variance but does not eliminate the structural supervision gap: hybrid supervision matches or exceeds human-only performance across scales.

\begin{table}[t]
\centering
\footnotesize
\begin{tabular}{c|c|cc}
\toprule
Model & N & Human-only & Hybrid \\
\midrule
\multirow{4}{*}{BERT-tiny}
 & 2000 & 0.373 & 0.411 \\
 & 4000 & 0.410 & 0.409 \\
 & 8000 & 0.429 & 0.427 \\
 & 16000 & 0.437 & 0.461 \\
\midrule
\multirow{4}{*}{DistilRoBERTa}
 & 2000 & 0.355 & 0.395 \\
 & 4000 & 0.406 & 0.409 \\
 & 8000 & 0.446 & 0.430 \\
 & 16000 & 0.441 & 0.447 \\
\midrule
\multirow{4}{*}{RoBERTa-base}
 & 2000 & 0.349 & 0.384 \\
 & 4000 & 0.397 & 0.409 \\
 & 8000 & 0.444 & 0.428 \\
 & 16000 & 0.457 & 0.465 \\
\bottomrule
\end{tabular}
\caption{Scaling with dataset size}
\label{tab:scaling}
\end{table}

\begin{figure}[t]
\centering
\includegraphics[width=0.48\textwidth]{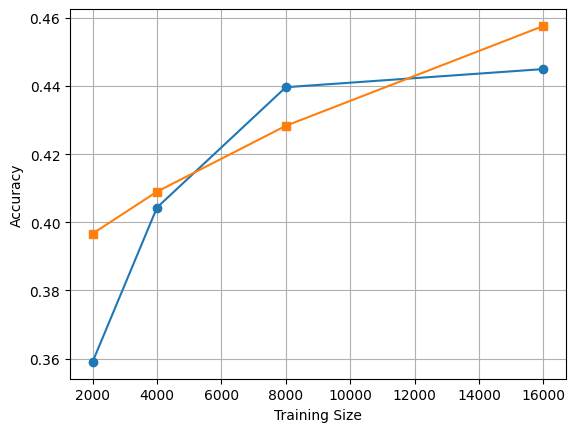}
\caption{Real-data scaling behavior. Pairwise accuracy versus training size for human-only supervision ($\alpha=1$, blue) and hybrid supervision ($\alpha=0.5$, orange). Hybrid supervision matches or exceeds human-only performance across scales, while scaling alone does not eliminate the structural supervision gap}
\label{fig:real-scaling}
\end{figure}

\subsection{Synthetic Known-Target Validation}

Synthetic experiments use a known reward $R^*(x,y)=w^\top \phi(x,y)$, allowing direct measurement of alignment error and distortion norms. As shown in Table~\ref{tab:synthetic} and Fig.~\ref{fig:synthetic-scaling}, distortion and alignment error increase monotonically toward human-only supervision ($\alpha=1$), confirming the predicted structural trajectory.

\begin{table}[t]
\centering
\footnotesize
\begin{tabular}{c|ccc}
\toprule
$\alpha$ & Accuracy & Alignment Error & Distortion Norm \\
\midrule
0.00 & 0.510 & 0.329 & 0.632 \\
0.25 & 0.513 & 0.534 & 0.670 \\
0.50 & 0.510 & 1.172 & 1.136 \\
0.75 & 0.502 & 2.797 & 2.658 \\
1.00 & 0.493 & 5.197 & 4.965 \\
\bottomrule
\end{tabular}
\caption{Synthetic validation: pairwise accuracy, alignment error, and distortion norm as a function of $\alpha$}
\label{tab:synthetic}
\end{table}

\begin{figure}[t]
\centering
\includegraphics[width=0.48\textwidth]{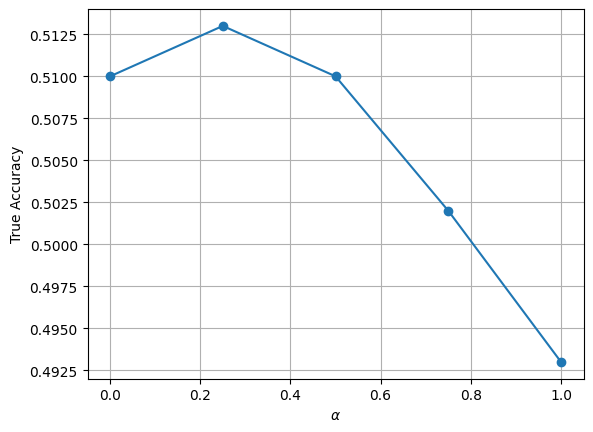}
\caption{Synthetic distortion trajectory. Objective accuracy as a function of the human-weight parameter $\alpha$ in the known-target synthetic task. Distortion increases monotonically toward human-only supervision ($\alpha = 1$), confirming the predicted structural alignment gap.}
\label{fig:synthetic-scaling}
\end{figure}

\subsection{Structural Supervision on GSM8K}

We evaluate HBI in an externally verifiable regime using GSM8K. Informative pairs consist of one correct and one incorrect solution generated by Mistral-7B-Instruct (4-bit). The auxiliary channel is defined as $S_A(x,y)=\mathbf{1}\{g(y)=Y^*(x)\}$, where $y$ is the generated completion, $g(y)$ extracts the predicted answer, and $Y^*(x)$ denotes the ground-truth target. Hybrid scoring is $S_\alpha = \alpha S_H + (1-\alpha) S_A$.

Objective accuracy as a function of $\alpha$ is reported in Table~\ref{tab:gsm8k-alpha}. Human-only supervision exhibits a persistent error floor, while hybrid supervision strictly dominates for all $\alpha<1$ and converges to perfect accuracy as $\alpha\rightarrow 0$, demonstrating auxiliary sufficiency.

\begin{table}[t]
\centering
\footnotesize
\begin{tabular}{c|c}
\toprule
$\alpha$ & Objective Accuracy \\
\midrule
0.00 & 1.000 \\
0.25 & 1.000 \\
0.50 & 0.983 \\
0.75 & 0.868 \\
1.00 & 0.696 \\
\bottomrule
\end{tabular}
\caption{Objective pairwise accuracy on GSM8K as a function of $\alpha$}
\label{tab:gsm8k-alpha}
\end{table}

\subsection{HumanEval: Auxiliary Sufficiency and Normalization Effects}

On \texttt{HumanEval} \cite{chen2021evaluating}, informative pairs consist of one functionally correct and one incorrect completion. The auxiliary channel encodes binary correctness $S_A(x,y) = \mathbf{1}\{\text{pass}\}$, while the human channel $S_H$ is a learned stylistic reward model. Hybrid scores use $s_{\text{hybrid}} = \alpha z(S_H) + (1-\alpha) z(S_A)$ with $\alpha=0.5$, where $z(\cdot)$ denotes batch $z$-score normalization. Across corruption levels $\gamma \in \{0,0.2,0.4\}$, human-only and hybrid supervision both yield $0.481 \pm 0.019$, while auxiliary-only achieves $1.000$, and Gaussian and shuffled null controls match human-only performance. This illustrates two boundary properties: (i) the human reward model exhibits a persistent structural floor relative to functional correctness, and (ii) auxiliary sufficiency collapses the floor when correctness is directly revealed.

The equality between human-only and hybrid performance arises from normalization-induced interaction suppression: because $S_A$ is binary and perfectly separable, batch $z$-score normalization removes auxiliary variance within pairs, causing convex mixing to preserve the ranking induced by $S_H$. The absence of hybrid improvement therefore reflects a normalization artifact rather than a contradiction of HBI.

Across real preference data, synthetic validation, and objective benchmarks (Tables~\ref{tab:alpha-summary}--\ref{tab:gsm8k-alpha}, Figs.~\ref{fig:real-scaling}--\ref{fig:synthetic-scaling}), results consistently support the central thesis: supervision-channel structure—not scale alone—determines error floors and distortion behavior, and auxiliary information expands effective learning capacity.

\section{Discussion and Conclusion}
\label{sec:discussion}

We developed a unified theory of human-supervised learning across six
frameworks—operator theory, PAC-Bayes, information theory, causal
inference, category theory, and a game-theoretic analysis of reinforcement
learning from human feedback—showing that when supervision is dominated
by the human channel $P_H$, learning is constrained by a strictly positive
excess-risk floor
\[
\gamma_H^{(\cdot)}
=
\gamma_{\mathrm{noise}}^{(\cdot)}
+
\gamma_{\mathrm{pref}}^{(\cdot)}
+
\gamma_{\mathrm{sem}}^{(\cdot)},
\]
arising from annotation noise, preference distortion, and semantic
compression. Although the numerical bounds differ across analytical
perspectives, their positivity reflects the same structural source:
non-sufficiency of the supervision channel. The limitation is
informational rather than architectural—scaling model size, data, or
compute cannot recover information that never passes through $P_H$.

Empirically, three regimes emerge. Human-only (H) supervision exhibits a
persistent floor. Hybrid human+model (H+M) supervision reduces variance
but retains structural distortions. When auxiliary channels provide
independent information about $Y^\ast$ (H+M+A), the floor weakly
decreases and collapses under sufficient information.

Experiments across real preference data, synthetic known-target tasks,
and externally verifiable benchmarks confirm these structural
signatures. In GSM8K, auxiliary correctness eliminates the floor,
instantiating the prediction that increased channel capacity reduces
distortion. In HumanEval, normalization-induced degeneracy shows that
observable hybrid gains require non-degenerate auxiliary variance.

The framework applies only when human supervision is the primary
information source about $Y^\ast$. Auxiliary channels—such as program
execution, retrieval, and verifiers—expand effective information
capacity and can remove the human bottleneck when they supply
independent signal. Hybrid systems therefore alter the supervision
channel itself rather than merely improving optimization within it.

Limitations include asymptotic optimization assumptions, semantic
abstraction, and mid-scale experiments. These affect convergence rates
but not the structural bound under information-reducing supervision.
Future work includes modeling tool-assisted supervision, estimating
$\gamma_{\mathrm{pref}}^{(\cdot)}$ and
$\gamma_{\mathrm{sem}}^{(\cdot)}$ in real datasets, extending the
information-theoretic analysis, and studying auxiliary-channel dynamics
in tool-integrated agents.

\bibliographystyle{IEEEtran}
\bibliography{ref}

\end{document}